\documentclass[11pt]{article}

\usepackage{arxiv}               
\usepackage[hyphens]{url}
\usepackage{graphicx}
\usepackage[round]{natbib}
\let\cite\citep                  
\usepackage[labelfont=bf,font=small]{caption}
\newtheorem{lemma}{Lemma}
\newtheorem{proposition}{Proposition}

\newenvironment{proof}[1][Proof]{\noindent\textit{#1.}\ }{\hfill$\square$\smallskip}

\usepackage[dvipsnames]{xcolor}
\usepackage[many]{tcolorbox}
\tcbuselibrary{breakable,skins}
\definecolor{promptframe}{RGB}{54,88,128}    
\definecolor{promptback}{RGB}{246,248,251}   
\definecolor{prompttitle}{RGB}{235,240,247}  
\newtcolorbox{promptbox}[1]{colback=promptback, colframe=promptframe,
coltitle=promptframe!20!black, colbacktitle=prompttitle,
boxrule=0.5pt, arc=1.5pt, left=5pt, right=5pt, top=4pt, bottom=4pt,
title={\small\bfseries #1}, fonttitle=\sffamily, before skip=6pt, after skip=6pt}
\newcommand{\phold}[1]{\textcolor{promptframe}{\itshape\{#1\}}}

\usepackage{algorithm}
\usepackage{algorithmic}
\usepackage{newfloat}
\usepackage{amsfonts}
\usepackage{bm}
\usepackage{amsmath}
\usepackage{listings}
\floatstyle{ruled}
\newfloat{listing}{tb}{lst}{}
\floatname{listing}{Listing}

\usepackage{booktabs}
\usepackage{multirow}

\usepackage[colorlinks=true,linkcolor=BrickRed,citecolor=NavyBlue,urlcolor=NavyBlue]{hyperref}

\title{CoRE: Consensus Rewards via Equilibrium for Test-Time Reinforcement Learning}
\shorttitle{CoRE: Consensus Rewards via Equilibrium}

\author{%
  Ambuj Mehrish \\
  CVML Lab \\
  Ca' Foscari University of Venice \\
  \texttt{ambuj.mehrish@unive.it} \\
  \and
  Sebastiano Vascon \\
  CVML Lab \\
  Ca' Foscari University of Venice \\
  \texttt{sebastiano.vascon@unive.it} \\
}
\date{}

\begin{document}

\maketitle

\begin{abstract}
On unlabeled test data, reinforcement learning lacks a ground-truth reward; test-time RL methods derive one from the model's own roll-outs, rewarding those that match the majority vote over $N$ sampled answers. That vote discards a correct answer whenever it is a minority and scores every majority-matching roll-out identically. We replace it with \emph{CoRE} (Consensus Rewards via Equilibrium): the $N$ roll-outs form a graph whose edges combine answer agreement, reasoning similarity, and generation confidence, and replicator dynamics extract its dominant set, yielding a refined pseudo-label, a graded per-roll-out reward, and a per-question cohesiveness gate. CoRE strictly generalizes voting: majority voting is recovered as a special case; a block-value analysis gives a sharp threshold for when consensus recovers a correct minority against a larger wrong plurality; and confidence calibration provably lowers that threshold multiplicatively. Across seven backbones and five benchmarks (42 model--benchmark cells, three seeds each), \emph{CoRE} improves the untrained base by $+21.7$ points on average versus $+20.4$ for majority-vote TTRL, wins wherever agreement is contestable with margins over the vote of up to $+7.5$ points, and reaches the voting baseline's plateau accuracy in $54$--$70$\% fewer steps. Consensus, not counting: treating the roll-out group as a graph rather than a ballot box turns a brittle vote into a calibrated, graded, self-supervised reward at no extra roll-out cost.
\end{abstract}

\section{Introduction}
\label{sec:intro}
Reinforcement learning has advanced language-model reasoning~\cite{guo2025deepseek,jaech2024openai,team2025kimi}, but typically requires labeled rewards~\cite{ouyang2022training,stiennon2020learning}. Test-time reinforcement learning (TTRL) instead derives supervision from the model's own roll-outs on unlabeled test questions~\cite{zuo2026ttrl,yuan2024self,zelikman2022star}. It uses the majority answer as a pseudo-label, rewards agreeing trajectories, and updates the policy. Although TTRL permits other aggregation rules, majority voting remains the default and can yield models that outperform the vote used to train them~\cite{shao2025spurious,xie2019self}.

This apparent paradox arises because model errors often scatter across multiple answers. Even when the majority label is wrong, many incorrect roll-outs disagree with it and still receive negative rewards, preserving part of the training signal. The failure mode is concentrated error: a systematic mistake becomes the plurality while the correct derivation remains in the minority. Majority voting then rewards the wrong trajectories, penalizes the correct ones, and can amplify the error across updates~\cite{arazo2020pseudo,gao2023scaling}.

Two additional limitations arise from the same reduction. First, a binary agreement reward assigns the same value to a well-supported derivation and a lucky guess. Second, each question contributes equally to training, regardless of whether its roll-outs form a coherent consensus or a fragmented set of incompatible solutions. Majority voting cannot distinguish these cases because it compresses each trajectory to a final answer and each answer class to a count~\cite{zuo2026ttrl}. It ignores both the model's confidence during generation and the relational structure among reasoning paths.

We address this limitation with \emph{CoRE}: Consensus Rewards via Equilibrium. \emph{CoRE} represents roll-outs as a graph whose edges connect trajectories with matching answers and weight their reasoning similarity by generation confidence. Replicator dynamics~\cite{pavan2007dominant} identifies a dominant set of mutually supporting trajectories, yielding a refined pseudo-label, graded roll-out rewards, and a cohesiveness gate for question-level updates. \emph{CoRE} integrates directly into GRPO~\cite{shao2024deepseekmath}, requires no labels or auxiliary models, and uses the same roll-outs as majority-based test-time reinforcement learning. Majority voting remains a special case of the formulation.

\emph{CoRE}'s dynamics begin from majority voting because uniform initialization assigns each answer class mass proportional to its size. A minority can therefore recover through coherence alone only if its internal support exceeds that of the majority by a factor that grows with the class-size imbalance. This predicts that graph structure rarely improves over voting by itself. Confidence-weighted voting is similarly limited: it rescales class counts but does not model support among reasoning trajectories.

The two signals are complementary. Confidence calibration rescales each answer class by its members' mean confidence, lowering the coherence needed for a correct minority to recover. Under our assumptions, this threshold decreases exponentially with the confidence gap between correct and incorrect roll-outs. Confidence can still fail when incorrect answers are systematically overconfident. Empirically, removing either confidence or graph structure reduces \emph{CoRE} to near-majority-vote performance, whereas combining them yields consistent gains.

We evaluate \emph{CoRE} on seven backbones from four developers, spanning math-specialized, base, and instruct models across AMC, AIME~2024, MATH-500 and its Level-4/5 subsets, and GPQA-Diamond, with three random seeds per setting. Across 42 model--benchmark pairs, \emph{CoRE} improves accuracy by $+21.7$ points over the untrained base, compared with $+20.4$ for majority-vote TTRL. It is the strongest RL variant in every model family, outperforming voting by up to $+7.5$ points when consensus is uncertain and remaining within seed variation when the majority is reliable. \emph{CoRE} also reaches the voting baseline's final accuracy in 54--70\% fewer optimization steps (Fig.~\ref{fig:efficiency}). These results suggest that \emph{CoRE} complements majority voting by recovering supervision from smaller, coherent, and confident correct groups that vote counts can miss.

\paragraph{Contributions.}
\begin{itemize}
\item \textbf{Method.} \emph{CoRE}, an equilibrium-based consensus reward for test-time RL that turns the same $N$ roll-outs into a refined pseudo-label, graded per roll-out rewards, and a question-level cohesiveness gate, with no auxiliary models or extra roll-outs.
\item \textbf{Theory.} Majority voting is a special case of \emph{CoRE}. A block-value analysis gives the threshold at which a correct minority overturns an incorrect plurality and shows that confidence calibration lowers it multiplicatively.
\item \textbf{Evidence.} Across 42 model--benchmark settings and seven backbones, \emph{CoRE} improves over the untrained base by $+21.7$ points on average, compared with $+20.4$ for TTRL (Tables~\ref{tab:cat-math}--\ref{tab:cat-vanilla}). It is the strongest RL method in every model family, exceeds voting by up to $+7.5$ points when agreement is contested, and reaches the vote's plateau in 54-70\% fewer steps (Fig.~\ref{fig:efficiency}). Both its gains and failures follow the regimes predicted by the analysis.
\end{itemize}
\section{Related Work}

\paragraph{Test-time adaptation and test-time RL.}
Test-time adaptation uses unlabeled inputs to update model behavior, either through auxiliary self-supervision or entropy minimization~\cite{sun2020test,liang2023comprehensive,wang2020tent}. For LLMs, test-time scaling improves accuracy with additional inference compute~\cite{snell2024scaling,brown2024large}, while TTRL converts majority votes over sampled roll-outs into reinforcement-learning rewards~\cite{zuo2026ttrl}. Recent label-free methods instead optimize semantic-cluster entropy, intrinsic entropy, self-certainty, or majority-based self-training~\cite{zhang2026right,prabhudesai2025maximizing,zhao2025learning,shafayat2025can}. These approaches extend earlier work on self-generated supervision, self-reward, and self-play~\cite{zelikman2022star,yuan2024self}, alongside unsupervised reasoning objectives and agreement-based rewards~\cite{xu2025genius,zhang2025co}. Related analyses suggest that entropy dynamics and robustness to imperfect rewards partly explain why such updates can succeed~\cite{cui2025entropy,shao2025spurious}. \emph{CoRE} likewise trains on unlabeled test questions, but improves the consensus aggregator that produces the reward rather than modifying the adaptation loop itself.

\paragraph{Aggregating multiple samples at inference.} Self-consistency marginalizes over reasoning paths by majority vote \cite{wang2022self}, and confidence-informed self-consistency (CISC) replaces uniform votes with a confidence-weighted vote, cutting the required number of paths "by over 40\% on average"
\cite{taubenfeld2025confidence}. The accuracy of such compound inference schemes scales predictably with the number of model calls \cite{chen2024more}. A parallel line reweights or reranks samples with a learned verifier rather than by vote, using outcome-based signals \cite{cobbe2021training}, process- versus outcome-based feedback \cite{uesato2022solving,lightman2024let}, and step-aware verifiers that score intermediate reasoning \cite{li2023making}. These operate only at inference; \emph{CoRE} instead turns aggregation into a training signal, and confidence-weighted voting is one of its ablation arms.

\paragraph{Dominant sets and replicator dynamics.} The Motzkin--Straus theorem relates the maximum clique problem to a quadratic optimization over the simplex~\cite{motzkin1965maxima}, later extended through regularization~\cite{bomze1997evolution}. Dominant sets generalize maximal cliques to weighted graphs and recover coherent clusters through replicator dynamics~\cite{pavan2007dominant,bulo2017dominant} (RD). This framework, drawn on evolutionary game theory and the Baum--Eagon inequality~\cite{taylor1978evolutionary,weibull1997evolutionary,baum1967inequality}, elegantly bridges together combinatorial optimization (clique search), optimization (maximization of a quadratic functional), and dynamical systems (stable points search). To our knowledge, \emph{CoRE} is the first method to use dominant-set extraction to construct rewards for reinforcement learning and for consensus reaching.

\paragraph{Policy optimization.} \emph{CoRE} builds on GRPO~\cite{shao2024deepseekmath}, following policy-gradient methods from PPO and RLHF to critic-free variants~\cite{schulman2017proximal,ouyang2022training,ahmadian2024back,li2023remax}. Since group normalization removes reward magnitude~\cite{liu2025understanding}, \emph{CoRE} applies its cohesiveness gate to the loss rather than the reward. This preserves the advantage estimate and reduces exposure to reward over-optimization~\cite{gao2023scaling,skalse2022defining,pan2022effects,amodei2016concrete}.

\section{\emph{CoRE}: Method and Analysis}
\label{sec:method}

\subsection{Problem Setup}
\label{sec:setup}
Let $\mathcal{D}$ be an unlabeled test set and $\pi_\theta$ a policy language model. Test-time reinforcement learning adapts $\pi_\theta$ on $\mathcal{D}$ itself: for each question $q$, the policy samples $N$ i.i.d.\ roll-outs $o_1,\dots,o_N \sim \pi_\theta(\cdot \mid q)$, a consensus function maps the group to a pseudo-label $y^{*}$ and per-roll-out rewards $\{R_i\}_{i=1}^{N}$, and the policy is updated by group-relative policy optimization (GRPO)~\cite{shao2024deepseekmath} on those rewards. TTRL \cite{zuo2026ttrl} instantiates the consensus function with majority voting: $y^{*}_{\mathrm{vote}} = \arg\max_a |\{i : \hat y_i \equiv a\}|$ and $R_i = \mathbb{1}[\hat y_i \equiv y^{*}_{\mathrm{vote}}]$, where $\hat y_i$ is the answer extracted from $o_i$. Since this self-generated signal is the only supervision the model receives, the consensus function determines what test-time RL can learn; it is our object of study, with the sampling procedure and the policy-optimization update rule held fixed.
\begin{figure*}
    \centering
    \includegraphics[width=\linewidth]{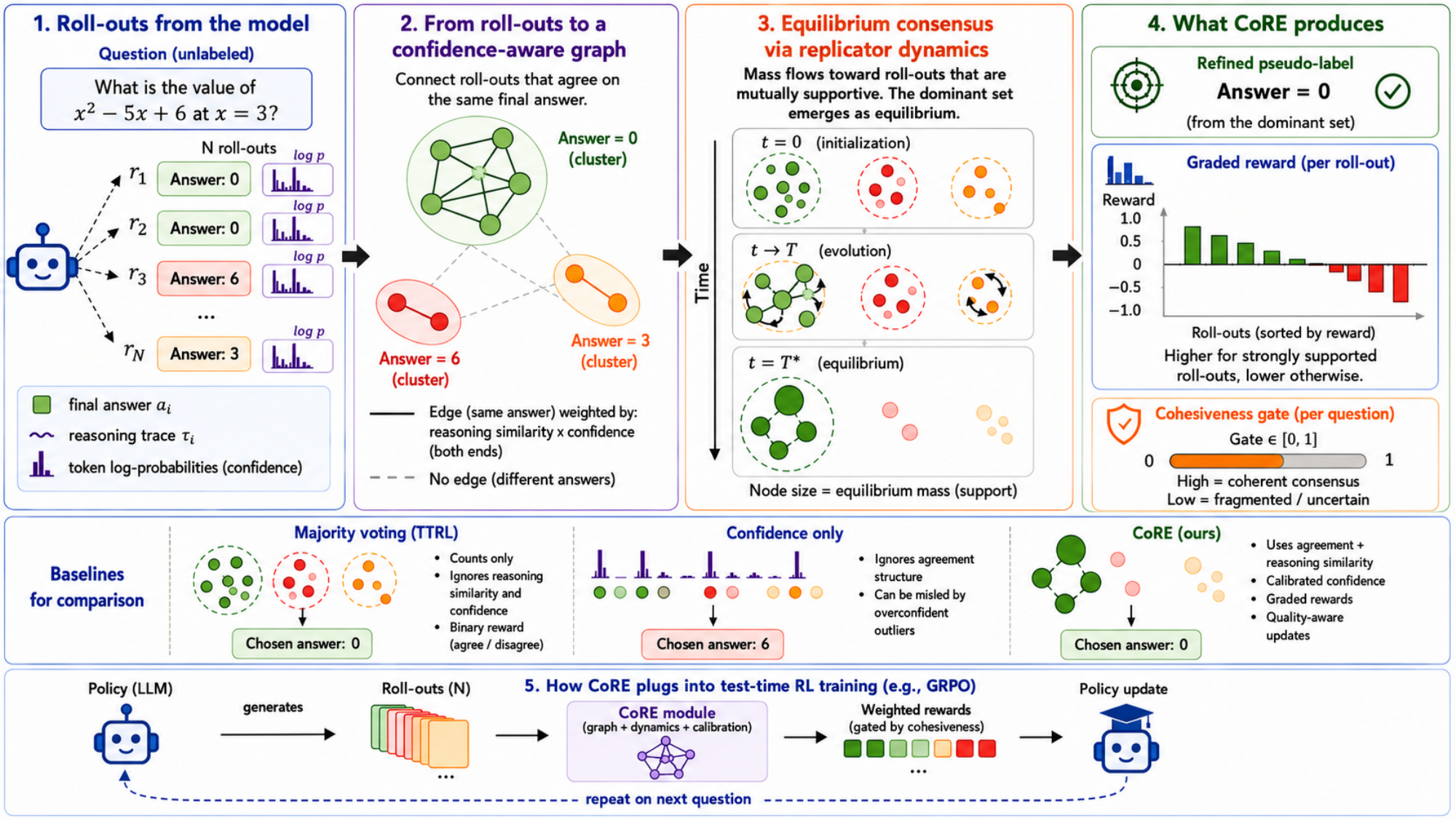}
    \caption{\textbf{CoRE overview.} Roll-outs become a graph whose edges combine answer agreement, reasoning similarity, and calibrated confidence; replicator dynamics extract the dominant set, yielding a refined pseudo-label, graded per-roll-out rewards, and a cohesiveness gate that plug directly into GRPO.}
    \label{fig:placeholder}
\end{figure*}
Each roll-out exposes three signals, all computable without model internals or external supervision: the extracted answer $\hat y_i$, compared by symbolic equivalence rather than string match; a reasoning representation $h_i \in \mathbb{R}^d$, a TF--IDF $n$-gram vector of the reasoning text~\cite{sparck1972statistical,salton1988term}; and a generation confidence $c_i = \tfrac{1}{|o_i|}\sum_t \log \pi_\theta(o_{i,t} \mid q, o_{i,<t})$, the mean token log-probability~\cite{kadavath2022language,tian2023just,lin2022teaching}. Majority voting consumes only the first, and only its mode. Throughout, the reward is computed solely from these signals; ground-truth answers are used only for final evaluation and for the post-hoc diagnostic analyses of \S\ref{sec:analysis}, and never enter the training
loop.
\subsection{Consensus Reward Construction}
\label{sec:core}
\emph{CoRE} replaces majority voting with dominant-set extraction over a roll-out graph that encodes answer agreement, reasoning similarity, and group cohesiveness. A dominant set generalizes the maximal clique concept to weighted graphs, allowing a compact and mutually supportive subset of trajectories to emerge from pairwise interactions. CoRE therefore produces three signals rather than TTRL's two: a refined pseudo-label $y^{*}$, a graded reward $R_i \in [0,1]$, and a question-level cohesiveness gate $w_q \in [0,1]$.


\paragraph{The roll-out graph}
The roll-out graph is graph $G=(V,E,\omega)$ where $V$ is the set of roll-out answers generated by the model and $E\subseteq V\times V$ is the set of edges connecting different nodes, weighted by the function $\omega_{i,j}: (i,j) \in E \rightarrow \mathbb{R_{\geq\text{0}}}$. The graph $G$ is encoded into an affinity matrix $A'$

\paragraph{Affinity graph.}
Given two nodes $i,j \in V$, two kernels compare the corresponding roll-outs: a hard answer kernel
$K_{\mathrm{ans}}(i,j) = \mathbb{1}[\hat y_i \equiv \hat y_j]$ and a soft
reasoning kernel
$K_{\mathrm{rsn}}(i,j) = \tfrac12\bigl(1+\cos(h_i,h_j)\bigr) \in [0,1]$.
They combine multiplicatively,
\begin{equation}
A_{ij} \;=\; K_{\mathrm{ans}}(i,j)\,
\bigl[\kappa + (1-\kappa)\,K_{\mathrm{rsn}}(i,j)\bigr],
\qquad A_{ii}=0,
\label{eq:affinity}
\end{equation}
so that roll-outs with different answers share \emph{no} edge regardless of how similar their prose is: $A$ is block-diagonal over answer classes, and reasoning similarity can strengthen ties only \emph{within} a class. The floor $\kappa$ guarantees agreeing roll-outs a minimum affinity even when
their derivations diverge lexically; $\kappa$ is the generalization knob, and at $\kappa=1$ the graph reduces to pure answer agreement (Prop.~\ref{prop:majority}).

\paragraph{Confidence calibration.}
Node weights $w_i = \exp\bigl((c_i - \max_j c_j)/\tau\bigr) \in (0,1]$ convert confidences to relative weights anchored at the most confident roll-out. They calibrate the edges symmetrically,
\begin{equation}
A'_{ij} \;=\; A_{ij}\,\sqrt{w_i w_j},
\label{eq:calib}
\end{equation}
A diagonal congruence $A' = D^{1/2} A D^{1/2}$, $D=\operatorname{diag}(w)$. The geometric mean is a symmetric choice. It preserves the block structure exactly, keeps an edge only when both endpoints are confident, and makes a cluster’s internal score scale linearly with its average confidence. Proposition~\ref{prop:confidence} uses this last property.

\paragraph{Avoid singletons}
To avoid singletons, self-loops are added to the calibrated rollout-graph $W = A' - \alpha I$. The self-competition term $-\alpha$ on the diagonal penalizes small clusters and, per Lemma~\ref{lem:singleton} (see Appendix), excludes singletons outright. Since the diagonal
is negative, we shift $W' = W + C\,e e^{\!\top}$ with $C \ge \alpha$ to obtain a nonnegative affinity matrix, where $e$ is the all-ones vector. The blocks external to the main diagonal are kept equal to 0 to avoid incoherent consensus.

\paragraph{Dominant Set Extraction}
Given a readout graph $G=(V,E,\omega)$, a DS is extracted by maximizing a standard quadratic assignment problem $x^{\!\top}Ax$ over the standard simplex. The optimization is performed by iterating the replicator dynamics
\begin{equation}
x(t+1)_i \;\leftarrow\; x(t)_i \,\frac{(W'x(t))_i}{x(t)^{\!\top}W'x(t)}
\label{eq:replicator}
\end{equation}
until convergence $||x(t)-x(t+1)|| < \epsilon \text{ or } T_{\max}\text{ iterations.}$
Where $x(t)_i$ is the $i$-th component of vector $x$ at time $t$. Eq \ref{eq:replicator} starts from the barycenter of the simplex $x(0) = e/N$, leaving the same chances to each node to get extracted as part of the DS.
At convergence, the dynamic system reaches a stable point, denoted by $x^{*}$. $x^{*}$ concentrates the mass on the most mutually supporting nodes; its support $\sigma = \{\,i : x^{*}_i > \varrho\,\}$ is the consensus. Three lemmas (Appendix, Lemmas~\ref{lem:shift}--\ref{lem:singleton}) establish that the construction is sound: the nonnegativity shift is invariant on the simplex, so $W$ and $W'$ share all maximizers and replicator fixed points (Lemma~\ref{lem:shift}); the update~\eqref{eq:replicator} is a Baum--Eagon growth transform, ascending monotonically and converging to fixed points whose supports are dominant sets (Lemma~\ref{lem:ascent}); and for $\alpha>0$ no singleton is ever selected when a genuine answer clique exists (Lemma~\ref{lem:singleton}).







\paragraph{Read-outs.}
Three signals are read from the equilibrium:
\begin{equation}
\begin{split}
y^{*} &= \operatorname{arg\,max}_{a}\,
{\textstyle\sum_{i:\,\hat y_i\equiv a}}\, x^{*}_i,\\
R_i &= \frac{(Ax^{*})_i}{\max_j (Ax^{*})_j},
\qquad
w_q = x^{*\top}\!A\,x^{*}.
\end{split}
\label{eq:readouts}
\end{equation}
The pseudo-label is the answer with the greatest equilibrium mass, so it may differ from the count-based plurality. Each roll-out is rewarded by its affinity to this consensus using the uncalibrated matrix $A$: confidence already shapes $x^{*}$, and using $A'$ again would double-count it. The question-level gate $w_q$ measures consensus coherence and scales the loss rather than the rewards, since GRPO normalization cancels any per-question reward scaling. It therefore leaves advantage estimation, clipping, and optimization unchanged, with $w_q \equiv 1$ recovering the standard objective. TTRL is the special case $(\kappa{=}1,\ w{\equiv}e,\ w_q{\equiv}1)$.
\subsection{When Does Consensus Beat the Vote?}
\label{sec:theory-sub}

Because the answer kernel gates all edges, the affinity matrix is block-diagonal: each answer class $V_a$ (size $n_a=|V_a|$) is a connected component, where $a \in \mathbb{A}$ is the set of possible answers. For $x\in\Delta$ write $s_a = \sum_{i\in V_a} x_i \leq 1$ for the mass on class $a$; the objective decomposes as $x^{\!\top} A' x = \sum_{a \in \mathbb{A}} s_a^2\, \varphi_a(x)$, where $\varphi_a$ is the class's internal coherence score, converging under the within-class dynamics to a limit $\mu_a$ (for a uniform clique of affinity $c$,
$\mu_a = c\,(1-1/n_a)$). Two questions arise: which class the \emph{global optimum} prefers, and which class the \emph{replicator dynamics, started from the barycenter, actually select} as the global maximizer lies in the class with the largest \emph{intensive} score $\mu_a$. 
The RD starts from the barycenter, allocating initial mass in proportion to class size; therefore the trajectory starts from the basin of the class with the largest \emph{extensive} score $n_a\mu_a$. The initialization literally encodes the ballot, leaving mass to the other possible outcome proportionally to their class size; coherence and confidence are what can overturn the convergence to the largest class size. All proofs are in Appendix~\ref{app:proofs}.


\begin{proposition}[Majority voting is a special case]\label{prop:majority} With answer-only affinity ($\kappa{=}1$) and uniform confidence ($w\equiv e$), the extracted class maximizes $n_a\mu_a = n_a - 1 - \alpha$, which is strictly increasing in $n_a$: the pseudo-label $y^{*}$ equals the plurality answer whenever the largest class is unique.
\end{proposition}

\begin{proposition}[Extraction threshold]\label{prop:separation} Let a correct clique of size $n^{*}$ and affinity $c^{*}$ compete with a wrong clique of size $m > n^{*}$ and affinity $c_w$ ($\alpha{=}0$). The dynamics started from the barycenter extract the correct class iff
\begin{equation}
c^{*} \;>\; \frac{m-1}{\,n^{*}-1\,}\; c_w .
\label{eq:threshold}
\end{equation}
\end{proposition}


\begin{proposition}[Confidence multiplies coherence]\label{prop:confidence} For an answer clique $S$ with base affinity $c$ and confidence weights $\{w_i\}$, the calibrated internal score obeys the exact identity
\begin{equation}
\mu'_S \;=\; c\bigl(\bar\mu_S^{\,2} - \bar\nu_S/|S|\bigr),
\qquad
\bar\mu_S = \operatorname{mean}_i \sqrt{w_i},\;\;
\bar\nu_S = \operatorname{mean}_i w_i,
\label{eq:identity}
\end{equation} so with $\bar w_S := \bar\mu_S^{\,2}$, the extraction threshold \eqref{eq:threshold} becomes
\begin{equation}
c^{*} \;>\; \frac{\bar w_w}{\bar w_{*}}\cdot\frac{m-1}{\,n^{*}-1\,}\; c_w .
\label{eq:calibthreshold}
\end{equation}
When correct roll-outs are more confident ($\bar w_{*} > \bar w_w$), the threshold drops multiplicatively.
\end{proposition}

Confidence is not merely a refinement but the enabling signal. Under temperature-$\tau$ weighting, a confidence gap $\delta=\bar c_{*}-\bar c_{w}$ in mean token log-probability reduces the recovery threshold by $e^{-\delta/\tau}$. Since $c_i$ uses natural logarithms, $\delta$ is measured in nats. At the gap observed in \S\ref{sec:analysis}, the $8$-vs-$24$ threshold falls from ${\approx}3.3,c_w$ to below $c_w$, turning minority recovery from unlikely to feasible. Confidence-weighted voting alone only rescales counts and cannot compare class-level reasoning coherence. The analysis therefore predicts that neither confidence nor graph structure is sufficient in isolation; their multiplicative interaction in \eqref{eq:calibthreshold} creates the recovery regime.
\paragraph{Predictions.}
The analysis makes falsifiable forecasts that \S\ref{sec:experiments} tests: (i) at $\kappa{=}1, w{\equiv}e$, \emph{CoRE}'s pseudo-labels match majority voting exactly (Prop.~\ref{prop:majority}); (ii) the graph alone and confidence-weighted voting alone should each perform on par with the vote, while their combination should not (Props.~\ref{prop:separation}--\ref{prop:confidence}); (iii) gains should concentrate where a coherent, confident correct minority exists to be recovered, vanish where the vote is already right or no correct roll-outs exist, and reverse where the confidence premise fails.
\begin{table*}[t]
\centering
\small
\setlength{\tabcolsep}{4pt}
\newcommand{\stdev}[1]{{\scriptsize$\pm$#1}}
\resizebox{\linewidth}{!}{%
\begin{tabular}{ll c cccc cccc}
\toprule
Model & Bench & Base & Majority & EC & CISC & CoRE & $\Delta_{\text{Maj}}$ & $\Delta_{\text{EC}}$ & $\Delta_{\text{CISC}}$ & $\Delta_{\text{CoRE}}$ \\
& & (no RL) & \emph{(TTRL)} & (graph) & (conf.) & (ours) & \multicolumn{4}{c}{arm $-$ base} \\
\midrule
Qwen2.5-Math-1.5B & AMC & 26.7 & 47.3\stdev{0.8} & 45.1\stdev{2.3} & 45.4\stdev{3.0} & 48.9\stdev{0.9} & $+20.7$ & $+18.4$ & $+18.7$ & $\mathbf{+22.3}$ \\
 & MATH-500 & 31.0 & 71.7\stdev{0.7} & 73.5\stdev{0.3} & 73.4\stdev{0.6} & 77.0\stdev{1.5} & $+40.7$ & $+42.5$ & $+42.4$ & $\mathbf{+46.0}$ \\
 & \quad -MATH-L4  & 31.5 & 60.9\stdev{2.6} & 61.5\stdev{2.0} & 62.8\stdev{1.4} & 65.9\stdev{2.0} & $+29.4$ & $+30.0$ & $+31.3$ & $\mathbf{+34.4}$ \\
 & \quad -MATH-L5 & 26.3 & 43.4\stdev{1.9} & 40.2\stdev{6.2} & 42.1\stdev{0.8} & 43.3\stdev{2.1} & $\mathbf{+17.1}$ & $+13.9$ & $+15.9$ & $+17.0$ \\
 & AIME$^{\P}$ & 5.6 & 20.0\stdev{5.5} & 18.9\stdev{4.2} & 17.8\stdev{1.6} & 18.9\stdev{4.2} & $\mathbf{+14.4}$ & $+13.3$ & $+12.2$ & $+13.3$ \\
 & GPQA$^{*}$ & 17.6 & 27.6\stdev{0.2} & 22.5\stdev{1.3} & 19.2\stdev{1.9} & 23.6\stdev{3.1} & $\mathbf{+10.0}$ & $+4.9$ & $+1.6$ & $+6.0$ \\
\addlinespace[2pt]
Qwen2.5-Math-7B & AMC & 39.9 & 67.5\stdev{1.0} & 67.1\stdev{0.6} & 66.7\stdev{2.0} & 67.5\stdev{0.0} & $\mathbf{+27.6}$ & $+27.2$ & $+26.8$ & $+27.6$ \\
 & MATH-500 & 47.9 & 85.0\stdev{0.2} & 84.3\stdev{0.8} & 85.6\stdev{0.2} & 84.7\stdev{0.4} & $+37.1$ & $+36.4$ & $\mathbf{+37.7}$ & $+36.8$ \\
 & \quad -MATH-L4 & 46.3 & 78.6\stdev{2.9} & 77.1\stdev{0.7} & 78.6\stdev{0.8} & 81.5\stdev{2.1} & $+32.3$ & $+30.8$ & $+32.3$ & $\mathbf{+35.2}$ \\
 & \quad -MATH-L5 & 36.1 & 59.2\stdev{1.5} & 61.0\stdev{1.2} & 62.1\stdev{0.3} & 61.2\stdev{0.8} & $+23.2$ & $+24.9$ & $\mathbf{+26.0}$ & $+25.1$ \\
 & AIME & 6.7 & 41.1\stdev{1.6} & 40.0\stdev{0.0} & 41.1\stdev{1.6} & 42.2\stdev{3.2} & $+34.4$ & $+33.3$ & $+34.4$ & $\mathbf{+35.6}$ \\
 & GPQA & 23.5 & 21.2\stdev{0.3} & 29.8\stdev{1.4} & 36.2\stdev{3.0} & 28.7\stdev{1.1} & $-2.3$ & $+6.3$ & $\mathbf{+12.7}$ & $+5.2$ \\
\addlinespace[2pt]
DeepSeek-Math-7B & AMC & 9.6 & 19.1\stdev{1.7} & 20.7\stdev{0.4} & 19.5\stdev{2.7} & 22.9\stdev{1.2} & $+9.5$ & $+11.1$ & $+9.9$ & $\mathbf{+13.4}$ \\
 & MATH-500 & 18.7 & 51.7\stdev{0.5} & 51.7\stdev{0.4} & 51.1\stdev{1.6} & 51.6\stdev{1.0} & $\mathbf{+33.0}$ & $+32.9$ & $+32.4$ & $+32.8$ \\
 & \quad -MATH-L4 & 15.2 & 36.6\stdev{2.5} & 36.2\stdev{2.6} & 38.5\stdev{2.3} & 40.1\stdev{3.1} & $+21.4$ & $+21.0$ & $+23.2$ & $\mathbf{+24.9}$ \\
 & \quad -MATH-L5 & 8.3 & 16.5\stdev{1.4} & 14.6\stdev{1.3} & 15.9\stdev{2.3} & 17.6\stdev{1.3} & $+8.2$ & $+6.3$ & $+7.6$ & $\mathbf{+9.3}$ \\
 & AIME & 2.6 & 1.7\stdev{1.7} & 2.2\stdev{1.6} & 1.7\stdev{1.7} & 4.4\stdev{1.6} & $-1.0$ & $-0.4$ & $-1.0$ & $\mathbf{+1.8}$ \\
 & GPQA & 30.0 & 32.1\stdev{1.6} & 33.0\stdev{0.8} & 33.3\stdev{1.0} & 32.3\stdev{1.0} & $+2.1$ & $+3.0$ & $\mathbf{+3.3}$ & $+2.3$ \\
\midrule
\multicolumn{7}{r}{\textbf{Mean} $\Delta$ vs.\ Base} & $+19.9$ & $+19.8$ & $+20.4$ & $\mathbf{+21.6}$ \\
\bottomrule
\end{tabular}}
\caption{\textbf{Math-specialized models.} Final pass@1, mean$\pm$std over $3$ seeds (\%). Majority = TTRL ($\kappa{=}1$, Prop.~1). $\Delta$ columns are each arm minus the no-RL base; the largest per-row improvement over base is \textbf{bold}; $^{*}$out-of-domain (1.5B near random floor, see limitations); $^{\P}$within seed noise.}
\label{tab:cat-math}
\end{table*}

\begin{table*}[t]
\centering
\small
\setlength{\tabcolsep}{5pt}
\newcommand{\stdev}[1]{{\scriptsize$\pm$#1}}
\resizebox{\linewidth}{!}{%
\begin{tabular}{ll c cccc cccc}
\toprule
Model & Bench & Base & Majority & EC & CISC & CoRE & $\Delta_{\text{Maj}}$ & $\Delta_{\text{EC}}$ & $\Delta_{\text{CISC}}$ & $\Delta_{\text{CoRE}}$ \\
& & (no RL) & \emph{(TTRL)} & (graph) & (conf.) & (ours) & \multicolumn{4}{c}{arm $-$ base} \\
\midrule
Qwen2.5-7B & AMC & 21.6 & 54.2\stdev{1.0} & 52.6\stdev{2.3} & 56.6\stdev{1.0} & 55.4\stdev{0.0} & $+32.6$ & $+31.0$ & $\mathbf{+35.0}$ & $+33.8$ \\
 & MATH-500 & 60.5 & 80.5\stdev{1.1} & 80.4\stdev{1.2} & 80.5\stdev{1.0} & 79.7\stdev{0.5} & $\mathbf{+20.0}$ & $+19.9$ & $+20.0$ & $+19.2$ \\
 & \quad -MATH-L4 & 31.0 & 75.0\stdev{2.2} & 74.5\stdev{2.2} & 75.5\stdev{1.6} & 77.9\stdev{2.4} & $+44.0$ & $+43.5$ & $+44.5$ & $\mathbf{+46.9}$ \\
 & \quad -MATH-L5 & 17.5 & 50.7\stdev{1.1} & 53.5\stdev{1.0} & 51.2\stdev{1.8} & 56.9\stdev{0.9} & $+33.2$ & $+36.0$ & $+33.7$ & $\mathbf{+39.4}$ \\
 & AIME & 2.2 & 23.3\stdev{2.7} & 21.1\stdev{5.7} & 23.3\stdev{0.0} & 23.3\stdev{0.0} & $\mathbf{+21.1}$ & $+18.9$ & $+21.1$ & $+21.1$ \\
 & GPQA & 20.5 & 29.2\stdev{0.5} & 28.7\stdev{1.2} & 33.5\stdev{2.4} & 29.5\stdev{4.0} & $+8.7$ & $+8.2$ & $\mathbf{+13.0}$ & $+9.0$ \\
\addlinespace[2pt]
Qwen3-8B & AMC & 58.3 & 71.3\stdev{1.3} & 68.9\stdev{1.0} & 70.6\stdev{0.4} & 73.3\stdev{2.0} & $+13.1$ & $+10.6$ & $+12.3$ & $\mathbf{+15.0}$ \\
 & MATH-500 & 82.3 & 90.3\stdev{0.1} & 88.9\stdev{0.4} & 90.4\stdev{0.3} & 88.6\stdev{0.3} & $+8.0$ & $+6.6$ & $\mathbf{+8.1}$ & $+6.3$ \\
 & \quad -MATH-L4 & 57.7 & 86.5\stdev{2.2} & 84.7\stdev{2.2} & 86.3\stdev{1.6} & 88.8\stdev{2.4} & $+28.8$ & $+27.0$ & $+28.6$ & $\mathbf{+31.1}$ \\
 & \quad -MATH-L5 & 45.0 & 67.9\stdev{1.1} & 69.2\stdev{1.0} & 66.1\stdev{1.8} & 71.6\stdev{0.9} & $+22.9$ & $+24.2$ & $+21.1$ & $\mathbf{+26.6}$ \\
 & AIME & 24.3 & 46.7\stdev{2.7} & 45.1\stdev{1.4} & 42.9\stdev{1.9} & 42.8\stdev{1.9} & $\mathbf{+22.4}$ & $+20.9$ & $+18.6$ & $+18.5$ \\
 & GPQA & 10.9 & 42.0\stdev{4.8} & 36.7\stdev{0.9} & 38.8\stdev{0.8} & 37.4\stdev{1.1} & $\mathbf{+31.1}$ & $+25.9$ & $+27.9$ & $+26.5$ \\
\midrule
\multicolumn{7}{r}{\textbf{Mean} $\Delta$ vs.\ Base} & $+23.8$ & $+22.7$ & $+23.7$ & $\mathbf{+24.5}$ \\
\bottomrule
\end{tabular}}
\caption{\textbf{Vanilla (base) models.} Measured entries report final pass@1 as mean$\pm$std over 3 seeds. Qwen3-8B is evaluated in non-thinking mode.}
\label{tab:cat-vanilla}
\end{table*}
\section{Experimental Setup}
\label{sec:experiments}
\paragraph{Models and benchmarks.} We evaluate $7$ backbones from four developers across three regimes: math-specialized Qwen2.5-Math-1.5B, Qwen2.5-Math-7B~\cite{yang2024qwen2}, and DeepSeek-Math-7B~\cite{deepseek-math}; vanilla Qwen2.5-7B~\cite{yang2024qwen2} and Qwen3-8B in non-thinking mode~\cite{yang2025qwen3}; and Llama-3.1-8B-Instruct~\cite{grattafiori2024llama} and Mistral-Nemo-Instruct~\cite{jiang2023mistral} (Tables~\ref{tab:cat-math}--\ref{tab:cat-instruct}). Each model is adapted and evaluated on five free-form mathematical-reasoning benchmarks: AMC, AIME~2024, MATH-500~\cite{hendrycks2021measuring}, and its L4 and L5 subsets, which isolate harder, high-disagreement cases where correct minorities are more likely to be outvoted. We additionally use the 198-question GPQA-Diamond subset~\cite{rein2023gpqa} as an out-of-domain probe covering expert-validated, "Google-proof" graduate-level physics, chemistry, and biology questions with four answer choices. Math-specialized models are OOD in both subject and format and remain near the $25\%$ random-guess baseline, making GPQA-Diamond a probe of the operating envelope rather than in-domain ability (\S\ref{sec:analysis}). Appendix~\ref{app:prompts} provides all prompt templates.
\paragraph{Training}Following TTRL~\cite{zuo2026ttrl}, each model is adapted on the unlabeled test benchmark. For every question, we sample $N{=}64$ roll-outs, derive a label-free reward from their consensus, and update the policy with GRPO~\cite{deepseek-math}; ground-truth answers are used only for final evaluation. We compare four reward rules under the same training loop: Majority, the standard TTRL vote and the $\kappa{=}1$ special case of our operator (Prop.~\ref{prop:majority}); EC, equilibrium consensus with graded rewards but no confidence; CISC, a pre-registered confidence-weighted voting control without the graph~\cite{taubenfeld2025confidence}; and \emph{CoRE}, the full confidence-calibrated method. We report pass@1 as mean@16 using 16 samples at temperature~$0.6$ and top-$p$~$0.95$, averaged over three seeds for each model--benchmark pair. Base denotes the frozen model without RL, and each $\Delta$ reports the corresponding arm minus Base.
\paragraph{\emph{CoRE} configuration.} We use one fixed setting across all models, benchmarks, and seeds: $N{=}64$ following TTRL, $\kappa{=}0.1$, $\tau{=}0.25$, $\alpha{=}0.1$, and $C{=}\alpha$.\footnote{$\alpha$ is a small positive self-penalty (\S\ref{sec:method}); $\kappa{=}0.1$ emphasizes reasoning similarity, whereas $\kappa{=}1$ recovers majority voting (Prop.~\ref{prop:majority}); and $\tau{=}0.25$ provides mild confidence calibration.} Replicator dynamics starts from the barycenter and stops at $\epsilon{=}10^{-6}$ or $T_{\max}{=}200$, with support threshold $\varrho{=}1/(10N)$. Each question uses independently fitted character 4-gram TF--IDF reasoning vectors,\footnote{TF--IDF is deterministic and model-free, while lexical overlap captures shared derivation steps without an auxiliary encoder.} and consensus adds negligible $O(N^2)$ computation relative to roll-out generation.
\section{Results}
\paragraph{Math-specialized models:} Table~\ref{tab:cat-math} reports the math backbones. \textsc{CoRE} improves the no-RL base by $+25.0$ points on average and is the best arm on every model. Against the reward that TTRL actually uses Majority vote, \textsc{CoRE} wins on the benchmarks where agreement is contestable and stays within seed noise elsewhere. On Qwen2.5-Math-1.5B it improves over Majority by $+5.3$ on MATH-500 and $+5.0$ on MATH-L4; on Qwen2.5-Math-7B by $+7.5$ on GPQA; and on DeepSeek-Math-7B by $+3.8$ on AMC. The gains concentrate on the harder benchmarks, where a correct but non-majority cluster exists to be recovered. The only benchmarks on which \textsc{CoRE} trails Majority is MATH-L5 on Qwen2.5-Math-1.5B and MATH-500 on both Qwen2.5-Math-7B and DeepSeek-Math-7B do so by at most $0.3$ points, within the $0.2$--$0.6$ across-seed standard deviation, and on models where Majority already scores above $85\%$ and $50\%$ so little recoverable headroom remains Figure~\ref{fig:offline} compares offline pseudo-label accuracy before policy updates. Across both Qwen2.5-Math models, \textsc{CoRE} outperforms majority voting at every roll-out count $N$, and the gap persists as $N$ increases. The downstream gains therefore arise from a better training signal rather than optimization noise.

\paragraph{Vanilla and instruct models:} The improvement is not specific to math-specialized backbones. Table~\ref{tab:cat-vanilla} reports the vanilla base models: on Qwen2.5-7B, \textsc{CoRE} beats Majority by $+6.2$ on MATH-L5, $+2.9$ on MATH-L4, and $+1.2$ on AMC (mean $\Delta$ vs base $+24.5$). Table~\ref{tab:cat-instruct} reports the instruct models: on LLaMA-3.1-8B, \textsc{CoRE} improves over Majority by $+4.2$ on MATH-L4 and $+5.1$ on GPQA, and on Mistral-Nemo by $+2.2$ on MATH-L5 and $+3.8$ on GPQA (mean $\Delta$ vs base $+19.1$). Across all three model families the picture is consistent: \textsc{CoRE} is the best RL arm on most benchmarks and improves over the no-RL base by roughly $+21.7$ points on average across all 42 cells (family means $+21.6$, $+24.5$, $+19.1$), spanning four model developers (Qwen, DeepSeek, Meta, Mistral) and three capability regimes.

\begin{table*}[t]
\centering
\small
\setlength{\tabcolsep}{5pt}
\newcommand{\stdev}[1]{{\scriptsize$\pm$#1}}
\resizebox{\linewidth}{!}{%
\begin{tabular}{ll c cccc cccc}
\toprule
Model & Bench & Base & Majority & EC & CISC & CoRE & $\Delta_{\text{Maj}}$ & $\Delta_{\text{EC}}$ & $\Delta_{\text{CISC}}$ & $\Delta_{\text{CoRE}}$ \\
& & (no RL) & \emph{(TTRL)} & (graph) & (conf.) & (ours) & \multicolumn{4}{c}{arm $-$ base} \\
\midrule
LLaMA-3.1-8B & AMC & 4.5 & 31.3\stdev{1.0} & 33.3\stdev{2.5} & 34.5\stdev{4.6} & 33.3\stdev{4.6} & $+26.8$ & $+28.8$ & $\mathbf{+30.1}$ & $+28.9$ \\
& MATH-500 & 48.6 & 65.5\stdev{1.0} & 55.5\stdev{0.9} & 55.7\stdev{2.8} & 65.2\stdev{2.0} & $\mathbf{+16.9}$ & $+6.9$ & $+7.1$ & $+16.6$ \\
& \quad -MATH-L4 & 5.9 & 57.3\stdev{2.1} & 56.8\stdev{2.4} & 59.7\stdev{1.3} & 61.5\stdev{0.9} & $+51.4$ & $+50.9$ & $+53.8$ & $\mathbf{+55.6}$ \\
& \quad -MATH-L5 & 4.8 & 27.2\stdev{0.9} & 27.0\stdev{1.4} & 27.2\stdev{1.1} & 28.2\stdev{2.4} & $+22.5$ & $+22.2$ & $+22.5$ & $\mathbf{+23.5}$ \\
& AIME & 1.1 & 7.8\stdev{4.2} & 12.2\stdev{3.2} & 17.8\stdev{1.6} & 10.0\stdev{2.7} & $+6.7$ & $+11.1$ & $\mathbf{+16.7}$ & $+8.9$ \\
& GPQA & 22.6 & 27.9\stdev{2.0} & 31.4\stdev{0.5} & 32.5\stdev{0.1} & 33.0\stdev{1.5} & $+5.3$ & $+8.8$ & $+9.9$ & $\mathbf{+10.4}$ \\
\addlinespace[2pt]
Mistral-Nemo-Instruct & AMC & 14.0 & 25.7\stdev{1.1} & 20.1\stdev{2.3} & 25.3\stdev{2.6} & 25.3\stdev{1.0} & $\mathbf{+11.7}$ & $+6.1$ & $+11.3$ & $+11.3$ \\
& MATH-500 & 40.8 & 62.5\stdev{1.0} & 49.9\stdev{0.9} & 51.1\stdev{2.8} & 60.2\stdev{2.0} & $\mathbf{+21.7}$ & $+9.1$ & $+10.3$ & $+19.4$ \\
& \quad -MATH-L4 & 36.1 & 54.4\stdev{1.4} & 52.6\stdev{1.0} & 54.7\stdev{0.7} & 55.7\stdev{0.3} & $+18.3$ & $+16.5$ & $+18.6$ & $\mathbf{+19.6}$ \\
& \quad -MATH-L5 & 13.3 & 21.1\stdev{0.5} & 19.9\stdev{2.2} & 21.7\stdev{1.4} & 23.3\stdev{0.9} & $+7.8$ & $+6.6$ & $+8.5$ & $\mathbf{+10.0}$ \\
& AIME$^{\S}$ & 0.7 & 2.2\stdev{1.6} & 1.1\stdev{1.6} & 0.0\stdev{0.0} & 0.0\stdev{0.0} & $\mathbf{+1.5}$ & $+0.4$ & $-0.7$ & $-0.7$ \\
& GPQA & 3.3 & 24.6\stdev{2.6} & 24.5\stdev{2.7} & 26.9\stdev{4.5} & 28.4\stdev{3.8} & $+21.2$ & $+21.2$ & $+23.6$ & $\mathbf{+25.1}$ \\
\midrule
\multicolumn{7}{r}{\textbf{Mean} $\Delta$ vs.\ Base} & $+17.7$ & $+15.7$ & $+17.6$ & $\mathbf{+19.1}$ \\
\bottomrule
\end{tabular}}
\caption{\textbf{Instruct \& reasoning models.} Measured entries report final pass@1 as mean$\pm$std over 3 seeds. $^{\S}$ denotes a competence-floor setting: Mistral-Nemo cannot solve competition-level AIME, and TTRL~\cite{zuo2026ttrl} also reports $0$.}
\label{tab:cat-instruct}
\end{table*}
\begin{figure}[t]
    \centering
    \includegraphics[width=0.7\linewidth]{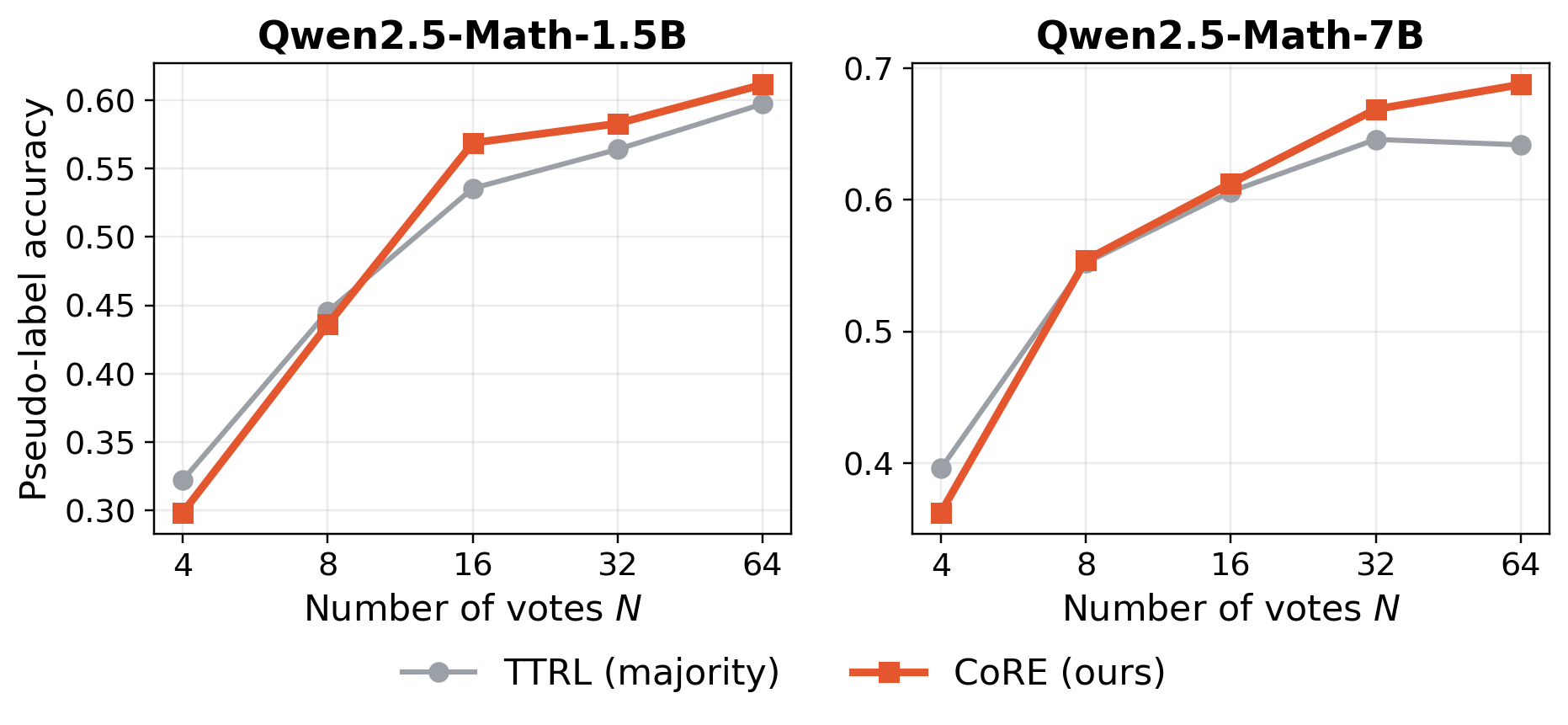}
    \caption{Offline pseudo-label accuracy versus roll-out count $N$. \textsc{CoRE} outperforms majority voting at every $N$, indicating a better training signal rather than optimization noise.}
    \label{fig:offline}
\end{figure}
\paragraph{Sample efficiency:} Figure~\ref{fig:efficiency} compares the validation trajectories of \textsc{CoRE} and TTRL on three model--benchmark pairs. \textsc{CoRE} reaches TTRL's plateau accuracy in $70\%$ fewer steps on LLaMA-3.1-8B (AMC), $57\%$ fewer on DeepSeek-Math-7B (AMC), and $54\%$ fewer on Qwen2.5-Math-1.5B (MATH). Unlike the binary majority reward, which assigns the same target to every consensus member, the equilibrium reward $R_i=(Ax^\ast)_i/\max_j(Ax^\ast)_j$ grades each roll-out by its support within the consensus. This denser signal appears to reduce gradient variance, linking faster convergence to the final accuracy gains.

\begin{figure}[t]
    \centering
    \includegraphics[width=\linewidth]{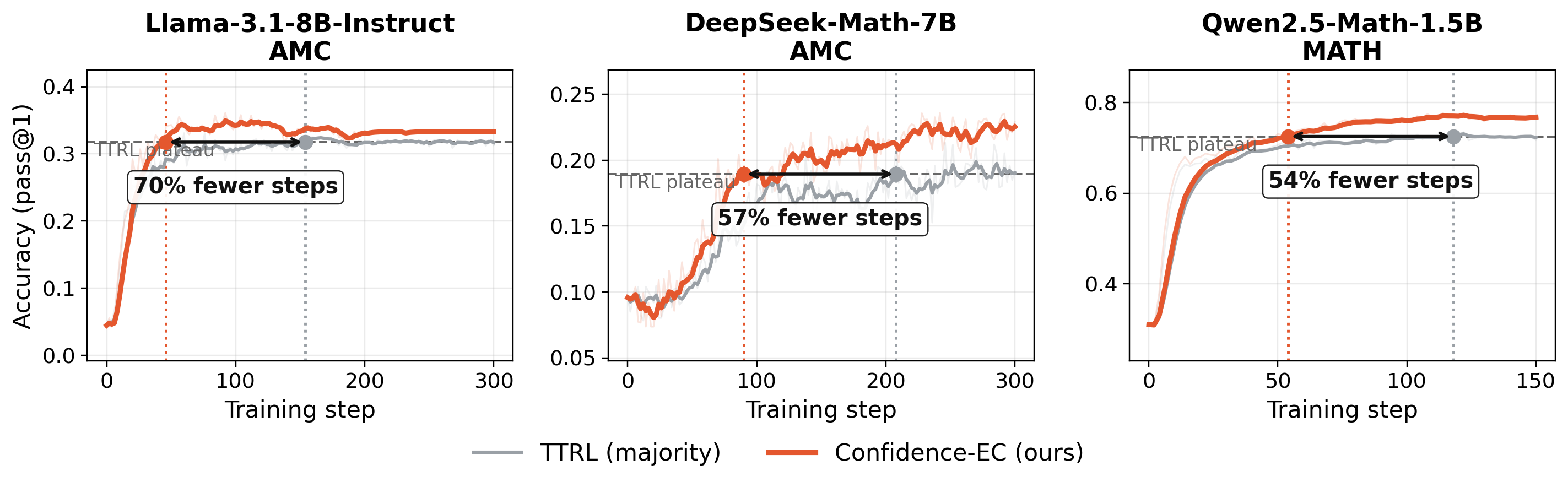}
    \caption{Validation pass@1 vs.\ training step (means over 3 seeds). \emph{CoRE} reaches TTRL's plateau accuracy in 54--70\% fewer steps.}
    \label{fig:efficiency}
\end{figure}

\begin{figure}[t]
    \centering
    \includegraphics[width=\linewidth]{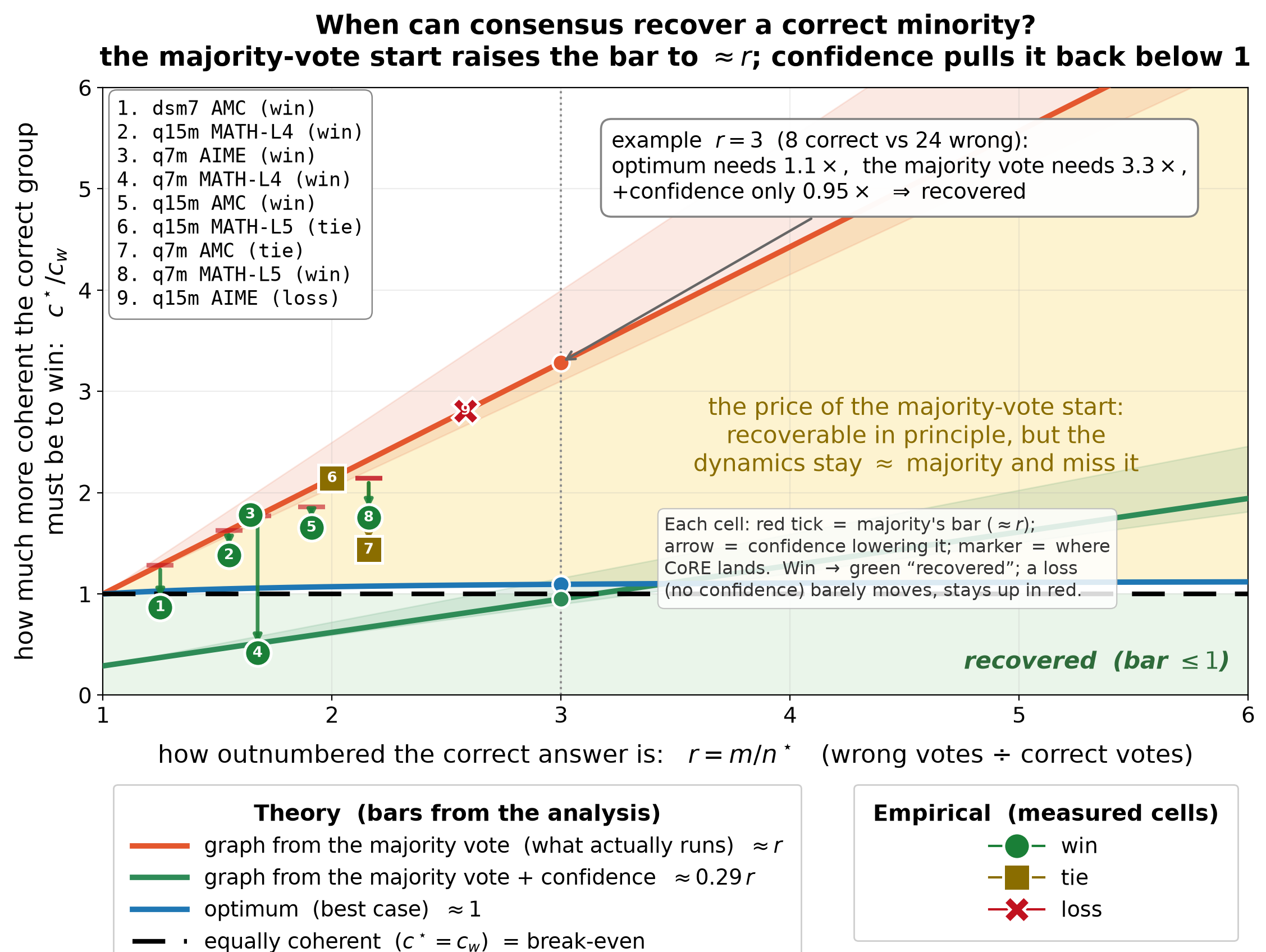}
    \caption{The recovery region: outnumbering ratio $r=m/n^{*}$ vs.\ required coherence advantage $c^{*}/c_w$. The majority-vote start raises the bar to ${\approx}r$ (red); confidence pulls it back by $e^{-\delta/\tau}$ (green). Measured cells below their bar are recovered wins; those above are losses.}
    \label{fig:recovery}
\end{figure}
\begin{figure}[t]
    \centering
    \includegraphics[width=0.9\linewidth]{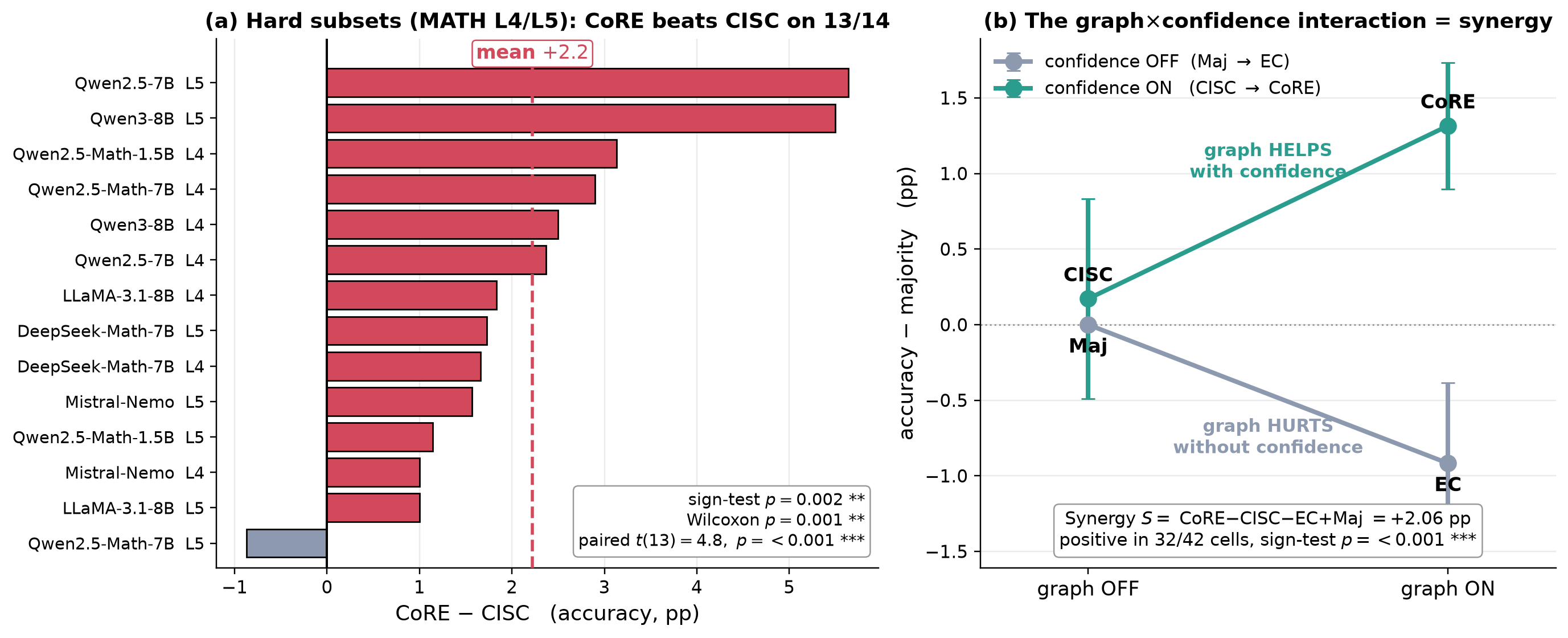}
    \caption{(a) On MATH L4/L5, \emph{CoRE} beats the confidence-only control CISC on 13/14 cells. (b) Graph and confidence interact synergistically: adding the graph to CISC helps, removing confidence from CoRE hurts.}
    \label{fig:synergy}
\end{figure}
\section{Analysis}
\label{sec:analysis}
\paragraph{\textsc{CoRE} vs.\ CISC.} CISC retains confidence weighting but removes the graph, equilibrium, and graded reward, so the $\textsc{CoRE}-\text{CISC}$ gap isolates the value of second-order structure. \textsc{CoRE} improves over CISC by $+1.2$ points on average, with family-level gains of $+1.2$, $+0.8$, and $+1.5$. On the contested MATH L4/L5 subsets, it wins 13 of 14 settings, with significance under both the sign test ($p{=}0.002$) and Wilcoxon signed-rank test ($p{=}0.001$) (Fig.~\ref{fig:synergy}a). Unlike confidence voting, which weights roll-outs independently, the quadratic objective $x^\top A x$ captures mutual support and can recover a coherent correct minority.
\paragraph{Recovery region.} Propositions~\ref{prop:separation}--\ref{prop:confidence} predict when a correct minority can be recovered, as visualized in Figure~\ref{fig:recovery}. The axes measure its size disadvantage, $r=m/n^\ast$, and required coherence advantage, $c^\ast/c_w$. The static optimum requires roughly equal coherence, whereas barycenter-initialized replicator dynamics raise the threshold to $\approx r$ because they begin from majority mass. Confidence calibration lowers it by $e^{-\delta/\tau}$, yielding $\approx 0.29r$ here. Model--benchmark pairs below this operative threshold are recoverable; those above it are not.
\paragraph{Operating regimes of consensus}The exceptions in Tables~\ref{tab:cat-math}--\ref{tab:cat-instruct} occur mainly at the \emph{competence floor}, where roll-outs lack coherent structure ($\chi \approx 0$) and no reliable minority cluster exists beyond the majority signal. This regime appears for Qwen2.5-Math-1.5B on out-of-domain GPQA and Mistral-Nemo on AIME, where all consensus methods remain near the underlying performance floor.At the opposite extreme, \emph{saturation} leaves little room for additional recovery. When a strong model produces nearly unanimous roll-outs, the majority ratio approaches one and the consensus rules become equivalent up to seed variation. The results on Qwen3-8B for GPQA and AMC are consistent with the $\kappa=1$ limit in Proposition~\ref{prop:majority}, under which \textsc{CoRE} reduces to Majority. Thus, the largest gains arise in the intermediate regime where the sample set contains a coherent correct minority that is not selected by majority voting alone.

The analysis also characterizes when confidence information is useful. Proposition~\ref{prop:confidence} assumes that correct clusters are, on average, assigned higher confidence than incorrect clusters. When this ordering is weak or reversed, confidence provides limited additional evidence. This motivates the combined use of confidence and graph structure in \textsc{CoRE}: each signal is most informative in a different part of the operating envelope, while the method naturally approaches Majority when no additional recoverable consensus is present.
\section{Limitations}
\label{sec:limitations}
\textsc{CoRE} helps only where roll-outs disagree and a coherent correct cluster exists (\S\ref{sec:analysis}): it cannot manufacture signal at the competence floor and reduces to majority under saturation. The reasoning kernel uses a
black-box lexical representation, white-box hidden states are a natural extension and we defer long-context ($32$k) long-reasoning-model training to future work (Appendix~\ref{app:prompts}).
\section{Conclusion}
We introduced \emph{CoRE}, which replaces majority voting in test-time reinforcement learning with the dominant-set equilibrium of a confidence-calibrated roll-out graph. Using the same $N$ roll-outs, \emph{CoRE} produces a refined pseudo-label, graded rewards, and a question-level cohesiveness gate, while retaining majority voting as a special case. Our analysis identifies when a coherent correct minority can overturn an incorrect majority and shows that confidence lowers this recovery threshold multiplicatively. Across seven backbones and five benchmarks, \emph{CoRE} improves over the untrained base by $+21.7$ points on average, compared with $+20.4$ for majority-vote TTRL, exceeds voting by up to $+7.5$ points, and reaches its final accuracy in 54--70\% fewer steps. The gains concentrate where agreement is contested, while the failures occur near the predicted competence floor. These findings suggest that the consensus operator is a substantive component of self-supervised RL: modeling mutual support among roll-outs yields a denser training signal than treating them as independent votes. Future work may replace lexical reasoning features with hidden-state representations and extend \emph{CoRE} to long-context reasoning models.
\section*{Acknowledgements}
This work was supported by the European Union's Horizon Europe research and
innovation programme under the Marie Sk\l{}odowska-Curie grant agreement
No.~101205348 (CASPER). We acknowledge the EuroHPC Joint Undertaking for
awarding this project access to the EuroHPC supercomputer LEONARDO, hosted by
CINECA (Italy) and the LEONARDO consortium, through the EuroHPC AI Factories
"AI for Science and Collaborative EU Projects" Access call (proposal
No.~EHPC-AIF-2026SC01-041). We further acknowledge the CINECA award under the
ISCRA initiative (Class C project IsCd5\_CASPER-A), for the availability of
high performance computing resources and support. Views and opinions expressed
are however those of the author(s) only and do not necessarily reflect those of
the European Union or the European Research Executive Agency. Neither the
European Union nor the granting authority can be held responsible for them.
\bibliographystyle{plainnat}
\bibliography{references}

\clearpage
\appendix
\section{Code and Reproducibility}
\label{app:code}
For reproducibility, we release the full codebase, including the roll-out graph construction, the replicator-dynamics consensus and read-outs, the GRPO training loop, and all evaluation scripts, in a
public repository:\\
\mbox{\url{TBA}}\\ The repository documents the environment and the configuration files (models, benchmarks, seeds, and the fixed hyperparameters of \S4) needed to reproduce every table and figure in this paper.

\section{Use of Large Language Models}
\label{app:llm}
Large language models were used only to improve grammar and clarity in author-written text.
\section{Well-Posedness of the Consensus Construction}
\begin{lemma}[Constant-shift invariance]\label{lem:shift}
For all $x$ on the simplex, $x^{\!\top}W'x = x^{\!\top}Wx + C$. Hence $W$ and $W'$ share all local and global maximizers and all replicator fixed points; the discrete dynamics on $W'$ differ from those on $W$ only in convergence rate, and continuous-time trajectories coincide exactly.
\end{lemma}

\begin{lemma}[Monotone ascent]\label{lem:ascent}
For entrywise-nonnegative symmetric $W'$, update \eqref{eq:replicator} is a Baum--Eagon growth transform: $x^{\!\top}W'x$ is non-decreasing at every step, strictly increasing off fixed points, and the iterates converge to fixed points whose supports are dominant sets of $W'$~\cite{pavan2007dominant}.
\end{lemma}

\begin{lemma}[No singleton consensus]\label{lem:singleton} For $\alpha>0$, a singleton has internal score $-\alpha<0$, while any answer clique $S$ with $|S|\ge 2$ and internal affinity $c > \alpha/(|S|-1)$ has positive internal score; hence no singleton is ever the consensus when such a
clique exists.
\end{lemma}
\section{Proofs}
\label{app:proofs}

Notation as in \S\ref{sec:method}: $\Delta$ is the standard simplex in
$\mathbb{R}^N$, $e$ the all-ones vector, and the payoff $W = A' - \alpha I$
is block-diagonal over answer classes because the answer kernel gates all
edges.

\section{Well-posedness of the construction}

The three lemmas of \S\ref{sec:core} establish that the consensus
computation is sound: the nonnegativity shift demanded by the dynamics costs
nothing, the iteration provably ascends and converges, and the equilibrium
can never be a lone roll-out.

\begin{proof}[Proof of Lemma~\ref{lem:shift}]
For $x \in \Delta$, $x^{\!\top}(C\,ee^{\!\top})x = C(e^{\!\top}x)^2 = C$, so
$x^{\!\top}W'x = x^{\!\top}Wx + C$: the objectives differ by a constant on
$\Delta$ and share all local and global maximizers. Fixed points also
coincide: $W'x = Wx + Ce$, so the stationarity condition---$(W'x)_i =
x^{\!\top}W'x$ on the support, $\le$ off it---holds for $W'$ iff it holds
for $W$, the constants cancelling. The discrete update on $W'$,
$x_i \mapsto x_i\,\frac{(Wx)_i + C}{x^{\!\top}Wx + C}$, moves in the same
direction as the update on $W$ with step size damped by $C$: iterates differ
along the way but converge to the same fixed points. In continuous time the
shift cancels exactly, $(W'x)_i - x^{\!\top}W'x = (Wx)_i - x^{\!\top}Wx$, so
the flows coincide. This licenses treating $W$ and $W'$ interchangeably in
everything that follows.
\end{proof}

\begin{proof}[Proof of Lemma~\ref{lem:ascent}]
For entrywise-nonnegative symmetric $W'$, $f(x) = x^{\!\top}W'x$ is a
degree-2 homogeneous polynomial with nonnegative coefficients and
$\partial f/\partial x_i = 2(W'x)_i$, so update \eqref{eq:replicator} is
exactly the Baum--Eagon growth transform of $f$ on $\Delta$, and $f$
strictly increases at every non-fixed point \cite{baum1967inequality}.
Continuity of $f$ on the compact $\Delta$ makes the non-decreasing sequence
$f(x(t))$ convergent, and strict increase off fixed points forces every
limit point of the iterates to be a fixed point. That the supports of strict
local maximizers are dominant sets is Theorem~1 of
\citet{pavan2007dominant}; we do not reprove it.
\end{proof}

\begin{proof}[Proof of Lemma~\ref{lem:singleton}]
For $x = e_i$: $x^{\!\top}Wx = W_{ii} = -\alpha < 0$ (zero-diagonal
affinity). For an answer clique $S$, $|S| = n \ge 2$, internal affinity $c$,
the uniform vector on $S$ attains
$x^{\!\top}Wx = \bigl(c(n-1) - \alpha\bigr)/n$, positive iff
$c > \alpha/(n-1)$. A negative-value point cannot be the maximizer while a
positive-value candidate exists.
\end{proof}

\section{Which class the dynamics extract}

The propositions of \S\ref{sec:theory-sub} concern which answer class the
replicator selects from the barycenter. The structural fact driving all of
them is block-diagonality: the flow decomposes into \emph{within-class}
refinement and a \emph{between-class} competition for mass. Two auxiliary
lemmas make this precise; the propositions then follow by analyzing the mass
competition.

\begin{lemma}[Block decomposition]\label{lem:decomp}
For $x \in \Delta$, write $s_a = \sum_{i \in V_a} x_i$ and, when
$s_a > 0$, $y^{(a)} = x^{(a)}/s_a$ for the within-class conditional. Then
$x^{\!\top}Wx = \sum_a s_a^2\,\varphi_a(x)$ with
$\varphi_a(x) = (y^{(a)})^{\!\top} W_a\, y^{(a)}$.
\end{lemma}

\begin{proof}
Block-diagonality kills every cross-class term; substituting
$x^{(a)} = s_a y^{(a)}$ in each block's quadratic form gives the claim.
\end{proof}

\begin{lemma}[Mass dynamics]\label{lem:mass}
Under the continuous-time replicator flow on $W$, for $i \in V_a$,
\begin{align}
\dot y^{(a)}_i &= y^{(a)}_i\bigl[(W_a y^{(a)})_i - \varphi_a(x)\bigr],
\label{eq:withinclass}\\
\dot s_a &= s_a\Bigl(s_a\varphi_a(x)
  - \textstyle\sum_b s_b^2\varphi_b(x)\Bigr).
\label{eq:mass}
\end{align}
Each within-class system \eqref{eq:withinclass} is itself a replicator dynamic
on $W_a$, so from a generic interior start
$\varphi_a(x(t)) \to \mu_a$, the internal score of the class's dominant
subset; for a uniform clique of affinity $c$ and size $n_a$,
$\mu_a = c\,(1 - 1/n_a) - \alpha/n_a$.
\end{lemma}

\begin{proof}
For $i \in V_a$, block-diagonality gives
$(Wx)_i = s_a (W_a y^{(a)})_i$; substituting into
$\dot x_i = x_i[(Wx)_i - x^{\!\top}Wx]$ with $x_i = s_a y^{(a)}_i$, summing
over $V_a$ for $\dot s_a$, applying the quotient rule for
$\dot y^{(a)}_i$, and using Lemma~\ref{lem:decomp} yields
\eqref{eq:withinclass}--\eqref{eq:mass}. Convergence of
\eqref{eq:withinclass} is Lemma~\ref{lem:ascent} applied to $W_a$; the
uniform-clique value is the computation in Lemma~\ref{lem:singleton}.
\end{proof}

Equation \eqref{eq:mass} is the object the paper's slogan refers to: the
classes play a coordination game for mass, and the barycenter start enters
through $s_a(0) = n_a/N$---the initialization encodes the ballot.

\begin{proof}[Proof of Proposition~\ref{prop:majority}]
At $\kappa{=}1$, $w \equiv e$, every class is a unit-affinity clique, so by
Lemma~\ref{lem:mass}, $\mu_a = (n_a - 1 - \alpha)/n_a$ and the within-class
conditionals start and, by symmetry, remain uniform:
$\varphi_a \equiv \mu_a$ throughout. In \eqref{eq:mass} each class grows iff
$s_a\mu_a$ exceeds the common mean $\sum_b s_b^2\mu_b$; moreover
$\frac{d}{dt}(s_a\mu_a) = (s_a\mu_a)\bigl(s_a\mu_a - \sum_b
s_b^2\mu_b\bigr)$, so the ordering of the quantities $\{s_a\mu_a\}$ is
preserved along the flow. At $t{=}0$,
$s_a(0)\mu_a = (n_a - 1 - \alpha)/N$ is uniquely maximized by the largest
class; its lead persists, all mass converges to it, and $x^{*}$ is uniform
on it, so the mass-weighted plurality $y^{*}$ equals the count-based one.
This is Prediction~(i) of \S\ref{sec:theory-sub}, which we also verify as a
unit test.
\end{proof}

\begin{proof}[Proof of Proposition~\ref{prop:separation}]
With two classes ($\alpha{=}0$), the conditionals stay uniform with constant
scores $\mu_{*} = c^{*}(1 - 1/n^{*})$, $\mu_w = c_w(1 - 1/m)$, and writing
$s := s_{*}$, \eqref{eq:mass} reduces to the one-dimensional system
\[
\dot s \;=\; s(1-s)\bigl[\mu_{*}\,s - \mu_w\,(1-s)\bigr].
\]
Both pure states are locally stable (at $s = 1-\varepsilon$,
$\dot\varepsilon \approx -(\mu_{*}{+}\mu_w)\varepsilon$; symmetrically at
$s=0$), and the interior fixed point
$s^{\mathrm{eq}} = \mu_w/(\mu_{*}{+}\mu_w)$ is unstable: the flow reaches
$s{=}1$ iff $s(0) > s^{\mathrm{eq}}$. The barycenter gives
$s(0) = n^{*}/(n^{*}{+}m)$, and
\[
\frac{n^{*}}{n^{*}+m} > \frac{\mu_w}{\mu_{*}+\mu_w}
\;\Longleftrightarrow\;
n^{*}\mu_{*} > m\,\mu_w,
\]
that is, $c^{*}(n^{*}{-}1) > c_w(m{-}1)$, which is \eqref{eq:threshold}. The
discrete update obeys the same threshold: with equilibrated scores the mass
ratio $\rho = s_{*}/s_w$ evolves as
$\rho(t{+}1) = (\mu_{*}/\mu_w)\,\rho(t)^2$, whose unstable point is
$\rho = \mu_w/\mu_{*}$---again $n^{*}\mu_{*} > m\,\mu_w$. We verified the
threshold to four decimal places on the exact pipeline of
\S\ref{sec:core}, including $\alpha = 0.1$ and the shift. The severity of
the factor $(m{-}1)/(n^{*}{-}1)$ is what makes coherence alone insufficient,
yielding Prediction~(ii).
\end{proof}

\section{What calibration changes}

Calibration acts on the internal scores that the mass competition compares.
The exact identity below yields the threshold rescaling
\eqref{eq:calibthreshold} and the $e^{-\delta/\tau}$ law quoted in
\S\ref{sec:theory-sub}.

\begin{proof}[Proof of Proposition~\ref{prop:confidence}]
For a clique $S$, $|S| = n$, with $A'_{ij} = c\sqrt{w_i w_j}$ off the
diagonal, the uniform vector on $S$ gives
\begin{align*}
\mu'_S
&= \frac{c}{n^{2}} \sum_{i \ne j} \sqrt{w_i w_j}
 = \frac{c}{n^{2}}\Bigl[\Bigl(\textstyle\sum_i \sqrt{w_i}\Bigr)^{2}
   - \textstyle\sum_i w_i\Bigr] \\
&= c\bigl(\bar\mu_S^{\,2} - \bar\nu_S/n\bigr),
\end{align*}
which is \eqref{eq:identity}; for constant weights $w_i \equiv w$ it reduces
to $w\,c\,(1 - 1/n)$: exact linear rescaling by $w$. Substituting the
rescaled scores $\mu'_{*} \approx \bar w_{*}\mu_{*}$ and
$\mu'_w \approx \bar w_w\mu_w$ (with $\bar w_S = \bar\mu_S^{\,2}$) into the
extraction condition $n^{*}\mu'_{*} > m\,\mu'_w$ of
Proposition~\ref{prop:separation} multiplies the threshold by
$\bar w_w/\bar w_{*}$, giving \eqref{eq:calibthreshold}; under the
temperature-$\tau$ weights a confidence gap of $\delta$ nats makes this
ratio $e^{-\delta/\tau}$. By Jensen's inequality
$\bar\mu_S^{\,2} \le \bar\nu_S$, so the exact score sits below the
first-order form by the finite-size correction $c\,\bar\nu_S/n$: calibration
is automatically conservative for small cliques, vanishing as $n$ grows.
\end{proof}
\providecommand{\armok}[3]{\fcolorbox{green!55!black}{green!10}{\strut \textbf{#1:}~#2~{\scriptsize(#3)}~\textcolor{green!45!black}{\checkmark}}}
\providecommand{\armno}[3]{\fcolorbox{red!55!black}{red!7}{\strut \textbf{#1:}~#2~{\scriptsize(#3)}~\textcolor{red!70!black}{$\times$}}}
\providecommand{\armcore}[2]{\fcolorbox{teal!65!black}{teal!12}{\strut \textbf{CoRE:}~#1~{\scriptsize(#2)}~\textcolor{teal!55!black}{\checkmark}}}

\section{Qualitative recovery analysis}
\label{app:qualitative}

We inspect where CoRE changes the pseudo-label relative to the majority vote, on Qwen2.5-Math-7B / MATH level~4 (128 questions, $N{=}64$ rollouts each). CoRE
recovers the correct answer on \textbf{10} questions that the majority vote gets wrong (29\% of the majority-wrong questions, all true minority overturns) against
only \textbf{4} regressions, a net $+6$. On the recovered questions the correct cluster is on average $+0.78$~nats more confident than the wrong plurality and
receives $+0.89$ more graded reward, and CoRE's cohesiveness gate averages
$\chi{=}0.42$.

Consistent with the factorial analysis (\S\ref{sec:analysis}), the two signals carry different cases. On most recoveries the wrong plurality is a
\emph{degenerate} cluster (empty or repetition-collapsed generations); there a confidence-weighted vote already suffices, so CISC recovers the answer too
(Example~\ref{ex:conf}). The graph becomes \emph{decisive} on the harder cases where the wrong answer is itself fluent, popular, and confidently stated. The
one question here where Majority, EC \emph{and} CISC all fail and only CoRE recovers (Example~\ref{ex:graph}). These are exactly the high-disagreement cases
the recovery-region theory targets.

\begin{tcolorbox}[breakable,colback=gray!2,colframe=teal!45!black,fonttitle=\bfseries,
  title={Example A: the graph is decisive (Majority, EC and CISC all wrong)}]
\label{ex:graph}
\textbf{Question.} \emph{With $\overline{ST}\parallel\overline{QR}$, $\angle P=40^\circ$
and $\angle Q=35^\circ$, find $\angle STR$.} \quad {\footnotesize(diagram omitted)}
\par\smallskip
\textbf{Ground truth:} $75^\circ$.\par\smallskip
\armno{Majority}{$35^\circ$}{14/64}\;\armno{EC}{$35^\circ$}{14/64}\;%
\armno{CISC}{$35^\circ$}{14/64}\;\armcore{$75^\circ$}{8/64}\par\smallskip
The popular wrong answer sets $\angle STR=\angle Q=35^\circ$ (using only one
angle); the correct minority applies $\angle STR=\angle Q+\angle P=35^\circ+40^\circ
=75^\circ$. Confidence is \emph{uninformative} here (gap $+0.05$~nats), so CISC
follows the majority --- only the graph's coherence picks out the correct
cluster. CoRE graded reward: $0.88$ (correct) vs.\ $0.00$; gate $\chi=0.51$.
\end{tcolorbox}

\begin{tcolorbox}[breakable,colback=gray!2,colframe=gray!55!black,fonttitle=\bfseries,
  title={Example B: a confidence-driven recovery (CISC also succeeds)}]
\label{ex:conf}
\textbf{Question.} \emph{Half the value of $3x-9$ is $x+37$. What is the value of
$x$?}\par\smallskip
\textbf{Ground truth:} $83$.\par\smallskip
\armno{Majority}{\footnotesize(no valid answer)}{5/64}\;%
\armno{EC}{\footnotesize(no valid answer)}{5/64}\;%
\armok{CISC}{$83$}{2/64}\;\armcore{$83$}{2/64}\par\smallskip
The wrong plurality is a degenerate cluster (repeated empty \verb|\boxed{}|
tokens). The correct minority solves $\tfrac12(3x-9)=x+37\Rightarrow x=83$ and is
much more confident (gap $+3.11$~nats), so both CISC and CoRE recover it. CoRE
graded reward: $1.00$ vs.\ $0.00$; gate $\chi=0.24$.
\end{tcolorbox}

\section{Prompt Templates}
\label{app:prompts}

All questions are presented as a single-turn chat message
\texttt{[\{"role": "user", "content": \textit{prompt}\}]}; the backbone's own
chat template is applied at roll-out time. The \emph{prompt} is the benchmark
question with a task-appropriate answer-format instruction appended, as
follows.

\paragraph{Mathematical reasoning (AMC, MATH-500, MATH Level~4/5, AIME
2024).}
The question is followed by the boxing instruction:

\begin{promptbox}{Math prompt template}
\ttfamily\small
\phold{question}\\[4pt]
Please reason step by step, and put your final answer within
\textbackslash boxed\{\}.
\end{promptbox}

\noindent Qwen2.5-Math backbones emit \verb|\boxed{}| natively and are run
with the bare question (no appended instruction); all non-Qwen-Math backbones
(LLaMA, Mistral, DeepSeek, vanilla Qwen) use the boxing instruction above
(the \texttt{\_instr} data variant), without which they box poorly and score
near zero.

\paragraph{Multiple-choice science (GPQA-Diamond).}
GPQA answers are letters (A--D). The question is followed by:

\begin{promptbox}{Multiple-choice prompt template}
\ttfamily\small
\phold{question}\\[4pt]
Give your reasoning, then output ONLY the letter of the correct option in
\textbackslash boxed\{\}, e.g.\ \textbackslash boxed\{A\}.
\end{promptbox}

\noindent This keeps a single \verb|\boxed{}| answer channel and one grader
across math and multiple-choice tasks (no separate MC parser).

\paragraph{Qwen3-8B (non-thinking mode).}
To match the non-thinking, 3k-context evaluation, the soft switch
\texttt{/no\_think} is appended to the user content before the boxing
instruction, so the model does not emit a \verb|<think>| block (equivalently,
the chat template is rendered with \texttt{enable\_thinking=False}).

\paragraph{No system prompt; sampling.}
No system message is used; the instruction lives in the user turn. Sampling for evaluation follows TTRL: temperature $0.6$, top-$p$ $0.95$, mean@16 ($16$ samples/question); maximum generation length $3072$ tokens (all non-LRM models) or $32{,}768$ tokens (LRMs: DeepSeek-R1-Distill, Skywork-OR1).

\end{document}